%% file: main.tex
\documentclass{article}

\usepackage[preprint]{neurips_2020}

\usepackage[utf8]{inputenc} 
\usepackage[T1]{fontenc}    
\usepackage{hyperref}       
\usepackage{url}            
\usepackage{booktabs}       
\usepackage{amsfonts}       
\usepackage{nicefrac}       
\usepackage{microtype}      
\usepackage{amsmath}
\usepackage{cleveref}
\usepackage{graphicx}
\usepackage{multirow}
\usepackage{amssymb}
\usepackage{amsthm}
\usepackage{wrapfig}

\input{math_commands.tex}

\Crefname{equation}{Eq.}{Eqns.}
\Crefname{figure}{Fig.}{Figs.}
\Crefname{theorem}{Thm.}{Thms.}
\creflabelformat{equation}{(#2#1#3)}
\newtheorem{lemma}{Lemma}
\newtheorem{theorem}{Theorem}

\newtheorem{definition}{Definition}

\title{Anomaly Detection with Domain Adaptation}

\author{%
  Ziyi Yang \\
  Stanford University \\
  \texttt{ziyi.yang@stanford.edu} \\
  \And
  Iman Soltani Bozchalooi \\
  Ford Greenfield Labs \\
  \texttt{isoltani@ford.com} \\
  \And
  Eric Darve \\
  Stanford University \\
  \texttt{darve@stanford.edu} \\
}

\begin{document}

\maketitle

\begin{abstract}
We study the problem of semi-supervised anomaly detection with domain adaptation. Given a set of normal data from a source domain and a \textbf{limited} amount of normal examples from a target domain, the goal is to have a well-performing anomaly detector in the target domain. We propose the Invariant Representation Anomaly Detection (IRAD) to solve this problem where we first learn to extract a domain-invariant representation. The extraction is achieved by an across-domain encoder trained together with source-specific encoders and generators by adversarial learning. An anomaly detector is then trained using the learnt representations. We evaluate IRAD extensively on digits images datasets (MNIST, USPS and SVHN) and object recognition datasets (Office-Home). Experimental results show that IRAD outperforms baseline models by a wide margin across different datasets. We derive a theoretical lower bound for the joint error that explains the performance decay from overtraining and also an upper bound for the generalization error.
\end{abstract}


\section{Introduction}

\vspace{-0.5em}

Also known as novelty detection or outlier detection, anomaly detection (AD) is the process of identifying abnormal items or observations that differ from what is defined as normal. Anomaly detection has been applied in many areas, including cyber security (detection of malicious intrusions), medical diagnosis (identification of pathological patterns), robotics (recognize abnormal objects), etc. Anomaly detection with different settings have been studied, for example, many anomaly detection works aim to solve the semi-supervised learning problem such that only normal data are available for training \citep{ruff2018deep, goad, rcgan}. The anomaly detection models are expected to learn an anomaly score function $A(\cdot)$ such that during testing anomalous data should be assigned higher anomalous scores than the examples labelled as ``normal.''

In practical applications, the normal data distribution can have a shift. For example, in manufacturing, we have sufficient amount of ``normal'' observations of engine type A (source domain) and we want to design an anomaly detection algorithm for a different but similar engine type B (target domain). However, we may have only \textbf{limited} normal observations for the target domain. One option is to re-collect a large-scale normal dataset in the new domain, however, this is often prohibitively costly and time-consuming for many practical applications, e.g., medical healthcare and autonomous driving \citep{ganin2015unsupervised, zhao19a}. Can we design a system that can leverage data from the both domains to learn an efficient anomaly detection model for the target domain? In this paper we attempt to address this important and interesting question. This type of problem is also known as Domain Adaptation (DA), which studies the transfer learning between the source and target domains \citep{dsn, gen2adapt, cycada, zhao19a}. 

Surprisingly, domain-adapted anomaly detection has not drawn as much interest as its classification peer, especially comprehensive studies on semi-supervised anomaly detection in the domain adaptation setting are rare. As an effort to solve the problem, we propose the Invariant Representation Anomaly Detection (IRAD) model. IRAD leverages a shared encoder to extract common features from source and target domain data. The shared encoder is adversarially trained with a source-domain specific encoder and a generator. Such design is required to avoid overfitting the target domain where training data are very limited. Then a simple and off-the-shelf anomaly detection model, Isolation Forest (IF), is trained on the extracted shared representations of source and target domain. At test time, the trained IF assigns the anomalous scores given the extracted features from the test target data.

We evaluate IRAD thoroughly on cross-domain anomaly detection benchmarks. Evaluation datasets are transformed from standard digit datasets (MNIST, USPS and SVHN) and Office-Home domain adaptation datasets. We compare IRAD to baselines including the prevailing anomaly detection models and competitive domain adaptation algorithms. Evaluation results show that IRAD outperforms the baseline models by significant margins. For example, on Office-Home dataset (Product$\rightarrow$Clip Art), IRAD improves upon the best baseline by almost 10\%. In addition, we derive a lower bound on the joint error on both domains for models based on invariant representations, which explains the observation that over-training the invariant feature extractor hurts the generalization to the target domain. We also obtain a generalization upper bound that reveals the sources of generalization error. We conduct ablation studies to confirm the effectiveness of objective functions in IRAD. 


\section{Related Work}

\vspace{-0.5em}

One major class of anomaly detection algorithms is generative models that learn the normal data distribution via generation processes, e.g., autoencoders (AE) and generative adversarial networks (GANs). \citet{adgan} train GANs \citep{gan} on images of healthy retina images to identify disease markers. Regularized Cycle-Consistent GAN \citep{rcgan} introduces a regularization distribution to correctly bias the generation towards normal data. Memory augmented generative models \citep{memae, magan} maintain external memory units that interact with the encoding process to store latent representations of the normal data. An emerging type of anomaly detection methods is self-supervised models \citep{geom, goad}. They first apply different transformations to the normal data and train a classifier to predict the corresponding transformation. The anomaly scores depend on the classifications predictions.

Domain adaptation is to learn from source domain data together with limited information of target domain in order to have a well-performing model on the target domain. One heavily studied direction is the unsupervised image classification. Given labeled source-domain images and unlabeled target-domain images, the goal is to obtain a target-domain classifier. One type of methods learns a transformation from the source to target domain \citep{cycada, lee2018diverse}; some approaches learn invariant representations between the two domains \citep{dsn, tzeng2017adversarial, zhao19a}. There are also works addressing few-shot domain adaptation with various problem settings. \citet{fsada} studies fully-supervised domain adaptation cases where target-domain data are limited and labeled. Few-shot domain translation \citep{NIPS2018_7480, biost} learns a mapping function from source to the target domain where limited target-domain data are given. Since the data setting is similar to IRAD, we include the one of state-of-the-art models BiOST \citep{biost} as a baseline.

Previous works studying the task of cross-domain anomaly detection typically have different problem setups from this paper. For example, most works assume access to labeled (both normal and abnormal) data at least in the source domain: \citet{chen2014transfer} learns a regressor using labeled source data, which can predict the target-domain anomaly distribution from the normal distribution (estimated from target-domain normal data). \citet{tad} learns the conditional data distribution with fully-supervised data in multiple target domains. A few works use only normal data in the source and target domain. \citet{ide2017multi} studies the collective anomaly detection problem, where the target domain is one of the source domains, with a mixture of Gaussian graphical models. \citet{adaflow} investigates the multi-source transfer anomaly detection problem by learning normalizing flows from target domains to the source domain. However, these works assume sufficient normal data in the target domain are available, which can be a demanding condition to meet as mentioned before.


\section{Methodology}

\vspace{-0.25em}

\textbf{Problem Statement} \;We investigate the problem of semi-supervised anomaly detection in the domain adaptation setting. For training, the learning algorithm has access to $n$ data points $\{(\vx^{(i)}_{src}, y_{src}^{(i)})\}_{i = 1}^n \in
(X \times Y)^n
$ sampled i.i.d.\ from the source domain $\gD_S$ and \textbf{limited} target data points $\{(\vx^{(j)}_{tgt}, y_{tgt}^{(j)})\}_{j = 1}^{n_t} \in 
(X \times Y)^{n_t}
$ sampled i.i.d.\ from the target domain $\gD_t$ (where $n_t$ is small and $n_t \ll n$). Let $y = 0$ ($y = 1$) denote normal (abnormal). In semi-supervised anomaly detection, we only have access to normal data, i.e., $y_{src}^{(i)} = 0$ and $y_{tgt}^{(j)} = 0$. The goal is to build an anomaly score function $A(\vx_{tgt}):
X \rightarrow a\in \R$ in the target domain. The test set consists of both normal and abnormal target domain data. An evaluation metric of learnt models is the area under the Receiver Operating Characteristic (ROC) curve, or AUROC, w.r.t.\ the true labels and anomaly scores of test examples. 




\subsection{Invariant Representations Extraction by Adversarial Learning}
Learning the domain-invariant features is a prevailing solution for the domain adaptation problem \citep{dsn, ganin2016domain, zhao19a}. Similar to \citet{dsn} and \citet{ganin2016domain}, IRAD includes a shared encoder $E_{sh}$ to extract common features between the source and target data. To enable an appropriate split of shared and domain-specific components, IRAD also has a private encoder $E_{pv}$ in the source domain to distill source-specific components. To ensure that the learned components actually contain useful information, we also introduce a generator to map from the latent space to data space in the source domain $G_{src}$. The generator $G_{src}$, encoders $E_{sh}$ and $E_{pv}$ are adversarially trained together using a discriminator $D_{src}$ in the source domain. The adversarial loss is given as follows:
\begin{equation}
\begin{split}
\min_{\{E_{sh}, E_{pv}, G_{src}\}}\max_{D_{src}} & V_{src}(D_{src}, G_{src}, E_{pv}, E_{sh}) = \\
&\E_{\vx_{src}}[\log D_{src}(\vx_{src})] + \E_{\vx_{src}}[\log(1 - D_{src}(\vx_{src}')] +\\
& + \E_{\vx_{src}, \vx_{tgt}}[\log(1 - D_{src}(\vx_{tgt}')] + \E_{\vx_{src}}[\log(1 - D_{src}(\vx_{rnd})]
\label{eq:adv}
\end{split}
\end{equation}

where $\vx_{src}' = G(E_{pv}(\vx_{src}) + E_{sh}(\vx_{src}))$ represents the reconstruction of the source data; $\vx_{tgt}' = G(E_{pv}(\vx_{src}) + E_{sh}(\vx_{tgt}))$ denotes the generation using the extracted common information $E_{sh}(\vx_{tgt})$ from the target data and private encodings from the source data; $\vx_{rnd} = G(\vz + E_{sh}(\vx_{src}))$ is generated using a variable $\vz$ sampled from a random distribution (empirically we find $\gN(0, 1)$ works well) together with shared encodings $E_{sh}(\vx_{src})$. The $\vx_{rnd}$ term is designed to avoid the scenario in which the private encoder is (incorrectly) so powerful that all latent information for the source domain is encoded with $E_{pv}$. By taking a random vector as part of the input, the shared encoder is trained adversarially to capture the essential information of the source data such that the generated $\vx_{rnd}$ is close to $\vx_{src}$. We conduct an ablation study about $\vx_{rnd}$ in \cref{sec:abl}. The discriminator $D_{src}$ is trained to distinguish real source data $\vx_{src}$ from $\vx_{src}'$, $\vx_{tgt}'$, and $\vx_{rnd}$. The shared encoder $E_{sh}$, $E_{pv}$, and $G_{src}$ are trained to maximize the error $D_{src}$ makes. At optimality, $\vx_{src}'$, $\vx_{tgt}'$ and $\vx_{rnd}$ should resemble real data $\vx_{src}$ w.r.t.\ $D_{src}$.
 
Besides adversarial training, we also optimize with the following cycle consistent losses:
\begin{align}
\label{eq:cycle}
l_{1} = \|\vx_{src} - \vx_{src}'\|_2 \qquad l_{2} = \|\vx_{src} - \vx_{tgt}'\|_2
\end{align}
The first loss enforces the cycle consistency property in the source data space. The second one ensures that components extracted from the target data $E_{sh}(\vx_{tgt})$ are actually shared features such that they reside in the same subspace as $E_{sh}(\vx_{src})$. The cycle consistency losses are crucial in our experiments with high-dimensional real-world images (e.g., in Home-Office dataset where image sizes are usually larger than $300\times300\times3$). We speculate this is due to the instability in GAN training for high-dimensional data \citep{wgan}. Stronger signals like direct cycle consistency losses should help the optimizations of generators and encoders.


\subsection{Split of Private and Shared Components}

The subspace of shared and private encodings of the source data should be dissimilar since they extract different features of $\vx_{src}$. For instance, in a domain adaptation problem on MNIST (source) and SVHN (target), denoted as MNIST$\rightarrow$SVHN, the shared encodings should learn to extract information relevant to the digit, while the private encodings are expected to contain components about digits style, size, etc. To enforce this characteristic, we introduce an optimization objective to minimize the similarity between the (normalized) shared encodings and private encodings, similar to \citet{dsn}:
\begin{align}
    l_{dis} &= \|E_{sh}(\vx_{src})^T E_{pv}(\vx_{src})\|
\label{eq:dis}
\end{align}
Also, the shared encodings extracted from two domains are expected to be similar, since they should capture the common information between the two domains. Therefore, we minimize the negative of inner product between the (normalized) shared encodings of the source and target data:
\begin{align}
    l_{sim} &= -\|E_{sh}(\vx_{src})^T E_{sh}(\vx_{tgt})\|
\label{eq:sim}
\end{align}
 We show in \Cref{fig:sim} that $l_{sim}$ objective is essential to ensure proximity between the shared encodings extracted from the source and target data. Without $l_{sim}$, we observe that the shared encodings of the source and target data are too far apart which undermines the performance of the anomaly detection algorithm. More details on this ablation study will be given in \cref{sec:abl}.

The final objective function of IRAD is the weighted sum of the losses mentioned above:
\begin{equation}
    V_{src} + \alpha (l_1 + l_2) + \beta (l_{dis} + l_{sim}) 
\end{equation}
Empirically, we find $\alpha = 1$, $\beta = 0.5$ works well and we use these values in all experiments.



\subsection{Anomaly Detection}
After the shared encoder is trained, we can conveniently leverage an off-the-shelf anomaly detection algorithm $A'$, Isolation Forest (IF, \citet{liu2008isolation}), to train an anomaly detection model using the shared representations extracted from both source and target data in the training set. The description of IF can be found in the next section. In general, any semi-supervised anomaly detection models can be used here. We choose IF because it is a prevailing and effective method with the standard implementation available \citep{scikit-learn}; also empirically we find IF works well. We conduct detailed comparisons between IRAD and vanilla IF in the experiments. For testing, given a test example $\vx$, we encode the test example $\vx$ to the shared subspace between source and target space $E_{sh}(\vx)$. The anomaly score $A(\vx)$ is then given as $A'(E_{sh}(\vx))$.

\section{Experimental and Theoretical Results}
\vspace{-0.5em}
We evaluate IRAD extensively on various kinds of domain adaptation benchmarks. The datasets include the MNIST digit images (source domain), SVHN and USPS. We also evaluate on the Clip Art and Product Images domains from the object recognition Office-Home Dataset \citep{office-home} with more realistic images of much larger sizes. Besides, we obtain the theoretical bounds for joint error and generalization error of IRAD, which are consistent with experimental observations. 

\subsection{Digits Anomaly Detection}

We test with two scenarios: the adaptation from MNIST (source) to USPS (target) and MNIST (source) to SVHN (target). An example of datasets setup is as follows. Assume digit 0 is the normal class. In the training phase, digit 0 from source domain (e.g., MNIST) as well as a limited number of digits 0 from the target domain (e.g., USPS) are available. Test data contain all the categories of digits in the target domain where digits 0 are labelled as ``normal'' and other digits are labelled as ``abnormal''. We use the original train/test split in the target dataset. In the digit anomaly detection experiments, the number of target training data $n_t$ 50. In the discussion section, we experiment with different numbers of target training data available. We compare IRAD with the following baselines:

\textbf{Isolation Forest (IF)} is a tree ensemble method that ``isolates'' data by randomly selecting a feature and then randomly selecting a split value between the maximum and minimum values of the selected feature to construct trees \citep{liu2008isolation}. The averaged path length from the root node to the example is a measure of normality. We experiment with two types of isolation forest: IF (T) trained with only target data; IF (S+T) trained with both source and target data.

\textbf{One Class Support Vector Machines (OCSVM)} is a classical anomaly detection algorithm similar to the regular SVM. OCSVM is a kernel-based method that learns a decision function for novelty detection \cite{ocsvm}. It classifies new data as similar or different to the normal data. Similar to IF, we test with two variants of OCSVM: OCSVM (T) and OCSVM (S+T).

\textbf{Bidirectional One-Shot Unsupervised Domain Mapping (BiOST)} is a recent work on few-shot domain transformation \citep{biost}. BiOST learns a encoder-generator pair for each domain respectively. Networks are then trained with across domains cyclic mapping losses and a KL divergence in the latent space similar to the one in Variational Autoencoder (VAE). The anomaly score of a target data example is its reconstruction error. BiOST is a representative baseline of methods that leverage cross-domain transformation \citep{cycada, adaflow}. 

\textbf{Deep Support Vector Data Description (DSVDD)} is a competitive deep learning one-class classification model for anomaly detection \citep{ruff2018deep}. DSVDD projects data to a sphere in the latent space by learning the feature encoder and the data center of the sphere. We train DSVDD on the union of source and target domain data.

We also test with an intuitive approach by augmenting the target domain training data, denoted as ``AGT''. The data are augmented by image rotations and flipping. An IF is then trained on the augmented data as the anomaly detector.

\begin{table}[htbp]
\caption{Anomaly detection with domain adaptation on MNIST (source domain) and USPS (target domain). The evaluation metric is AUROC in percent. The highest numbers are in bold.}
\resizebox{\columnwidth}{!}{\begin{tabular}{c|cccccccc}
\toprule
Model & DSVDD & OCSVM (S+T) & OCSVM (T) & IF (S+T) & IF (T) & AGT & BiOST & \textbf{IRAD} \\
\midrule
0 & 36.5$\pm$10.1 & 74.5$\pm$0.6 & 4.3$\pm$1.1 & 22.9$\pm$1.4 & 95.1$\pm$0.4 & 94.3$\pm$1.7 & 75.9$\pm$4.2 & \textbf{96.0$\pm$0.6}\\
1 & 74.3$\pm$7.1 & 4.0$\pm$0.2 & 1.1$\pm$0.1 & 97.3$\pm$0.6 & 98.7$\pm$0.1 & \textbf{99.0$\pm$0.1} & 97.6$\pm$1.8 & \textbf{99.0$\pm$0.2} \\
2 & 57.2$\pm$4.7 & 45.1$\pm$0.8 & 18.5$\pm$2.4 & 48.6$\pm$2.8 & 74.0$\pm$1.8 & 45.9$\pm$3.3 & 72.8$\pm$9.6 & \textbf{82.2$\pm$0.7}  \\
3 & 59.5$\pm$9.0 & 61.2$\pm$1.1 & 7.2$\pm$1.3 & 39.3$\pm$1.8 & 84.5$\pm$2.2 & 56.1$\pm$3.2 & 69.7$\pm$11 & \textbf{88.7$\pm$2.6}\\
4 & 68.3$\pm$9.0 & 20.6$\pm$0.6 & 10.8$\pm$1.4 & 71.5$\pm$0.4 & 81.3$\pm$1.4 & 73.2$\pm$1.7 & 79.1$\pm$10 & \textbf{88.3$\pm$1.4} \\
5 & 48.7$\pm$3.4 & 66.0$\pm$0.4 & 18.7$\pm$2.0 & 32.3$\pm$0.9 & 70.0$\pm$1.6 & 41.2$\pm$2.0 & 79.5$\pm$3.6 & \textbf{81.1$\pm$2.7}\\
6 & 65.1$\pm$6.1 & 42.3$\pm$0.7 & 5.6$\pm$2.3 & 59.4$\pm$1.4 & 95.7$\pm$0.8 & 65.3$\pm$3.8 & 90.0$\pm$4.0 & \textbf{96.3$\pm$1.0}\\
7 & 62.7$\pm$5.0 & 37.8$\pm$1.6 & 6.0$\pm$1.0 & 61.5$\pm$1.6 & 91.8$\pm$1.3 & 69.4$\pm$2.3 & 66.8$\pm$6.6 & \textbf{95.6$\pm$1.6} \\
8 & 53.1$\pm$10.6 & 46.4$\pm$0.5 & 10.3$\pm$2.2 & 51.0$\pm$1.4 & 79.1$\pm$1.3 & 68.7$\pm$3.8 & 78.3$\pm$9.5 & \textbf{83.7$\pm$2.3} \\
9 & 62.7$\pm$4.4 & 28.1$\pm$0.7 & 6.0$\pm$0.9 & 69.6$\pm$1.6 & 93.1$\pm$0.8 & 76.4$\pm$2.3 & 84.7$\pm$11 & \textbf{94.9$\pm$0.5}
\end{tabular}}
\label{tab:m-u}
\vspace{-1em}
\end{table}

\begin{table}[htbp]
\caption{Results on MNIST$\rightarrow$SVHN with highest numbers in bold. The metric is AUROC in percent.}
\centering
\resizebox{\columnwidth}{!}{\begin{tabular}{c|cccccccc}
\toprule
Model & DSVDD & OCSVM (S+T) & OCSVM (T) & IF (S+T) & IF (T) & AGT & BiOST & \textbf{IRAD} \\
\midrule
0 & 51.3$\pm$1.3 & 50.4$\pm$0.1 & 47.1$\pm$0.2 & 49.0$\pm$0.5 & 54.4$\pm$0.4 & 53.2$\pm$0.5 & 56.1$\pm$1.7 & \textbf{56.7$\pm$2.0}\\
1 & 51.7$\pm$1.2 & 49.5$\pm$0.1 & 51.0$\pm$0.2 & 50.4$\pm$0.4 & 51.8$\pm$0.6 & 50.7$\pm$0.5 & 56.7$\pm$1.8 & \textbf{61.0$\pm$2.0} \\
2 & 51.0$\pm$0.8 & 49.0$\pm$0.1 & 49.6$\pm$0.1 & 50.9$\pm$0.4 & 51.2$\pm$0.2 & 50.6$\pm$0.3 & 53.8$\pm$1.0 & \textbf{56.0$\pm$0.2}  \\
3 & 51.7$\pm$0.3 & 48.8$\pm$0.1 & 49.7$\pm$0.3 & 51.0$\pm$0.1 & 50.6$\pm$0.3 & 50.0$\pm$0.4 & 53.7$\pm$1.1 & \textbf{55.8$\pm$0.9}\\
4 & 50.4$\pm$0.7 & 49.7$\pm$0.1 & 51.3$\pm$0.3 & 49.7$\pm$0.5 & 51.7$\pm$0.5 & 50.9$\pm$0.5 & 54.7$\pm$1.5 & \textbf{55.9$\pm$1.1} \\
5 & 50.5$\pm$0.4 & 49.8$\pm$0.1 & 49.5$\pm$0.2 & 42.9$\pm$0.2 & 51.3$\pm$0.4 & 51.0$\pm$0.2 & 53.9$\pm$1.1 & \textbf{54.1$\pm$0.8}\\
6 & 49.2$\pm$0.5 & 50.8$\pm$0.1 & 49.8$\pm$0.2 & 48.8$\pm$0.3 & 52.4$\pm$0.5 & 51.3$\pm$0.4 & 55.9$\pm$1.4 & \textbf{56.6$\pm$1.4}\\
7 & 50.4$\pm$1.1 & 50.1$\pm$0.2 & 51.3$\pm$0.2 & 50.1$\pm$0.5 & 51.9$\pm$0.4 & 49.5$\pm$0.6 & 56.3$\pm$2.0 & \textbf{57.0$\pm$1.3} \\
8 & 50.5$\pm$2.3 & 50.5$\pm$0.2 & 50.0$\pm$0.3 & 49.1$\pm$0.3 & 51.3$\pm$0.2 & 50.7$\pm$0.3 & 53.0$\pm$0.8 & \textbf{54.2$\pm$0.9} \\
9 & 49.8$\pm$0.4 & 50.7$\pm$0.2 & 48.8$\pm$0.3 & 49.2$\pm$0.3 & 52.3$\pm$0.4 & 51.7$\pm$0.5 & 54.4$\pm$1.1 & \textbf{55.9$\pm$1.4}
\end{tabular}}
\label{tab:m-s}
\vspace{-0.5em}
\end{table}

Results for MNIST$\rightarrow$USPS and MNIST$\rightarrow$SVHN are presented in \Cref{tab:m-u} and \Cref{tab:m-s} respectively (averaged over 10 runs). IRAD outperforms all baseline models in both domain adaptation settings. An interesting observation is that Isolation Forest trained only on target data (IF(T)) makes a strong baseline and is even better than both transformation-based deep-learning BiOST (with USPS as the domain) and IF trained on both source and target data. We speculate that this is because USPS anomaly detection itself is not an over-challenging task and IF model is highly data efficient such that training with limited target data can produce strong performance. One explanation for the downgraded performance of IF (S+T) (compared with IF(T)) is that the MNIST and USPS digits are from close but still distinct distributions. MNIST data actually add noises to the training of IF and undermine the performance. We provide a theoretical explanation for this observation in \Cref{sec:thm}.

Images are preprocessed into gray scale single-channel images of size $32\times 32$ so that they can be input to the same network. The shared encoder, private encoder and discriminator in IRAD follow the configurations in standard DCGANs \citep{dcgan}. To ensure a fair comparison, we use the same neural network architectures in IRAD, BiOST and DSVDD (the feature extractor). Hyperparameters are chosen by cross-validation, e.g., the size of latent representation output by encoder is 64.

\subsection{Objects Recognition Anomaly Detection}
The Office-Home objects recognition dataset \citep{office-home} is a prevailing and challenging domain adaptation benchmark. The images are high-dimensional where the average side length is more than 300. We experiment with the Clip Art and Product domains. We evaluate on ten categories that have reasonably sufficient data for evaluation in both domains, as listed in the first column in \Cref{tab:p-c}. Object examples are shown in the appendix. We test on two experimental scenarios: Product$\rightarrow$Clip Art and Clip Art $\rightarrow$Product. Since the number of images in a domain is limited, we augment the training data in the source domain by rotations and flipping, which increases the size of training source data by eight times. For a fair comparison, baseline models are also trained with the augmented datasets. The number of target domain images in the training set $n_t = 10$.

The encoders $E_{sh}$ and $E_{pv}$ in IRAD are ResNet-50 \citep{resnet} pretrained on ImageNet where the last layer is removed and a fully connected layer is added. The decoder is a ten-layer transpose convolution neural networks. The discriminator $D_{src}$ is a ResNet-18 network (without pretraining) followed by a final layer for classification. To improve the optimization process, we use the adversarial objective as in least-square GANs \citep{lsgan}. To have a fair comparison, baseline models with encoding networks, e.g., DSVDD and BiOST, also leverage the pretrained ResNet-50 as the encoders. The size of latent representation output by encoder is 128 chosen by cross-validation.

Experimental results are shown in \Cref{tab:p-c} and \Cref{tab:c-p}, averaged on 10 runs. In each row, we regard the corresponding objects category as the normal class. IRAD shows strong performance in both adaptation scenarios and outperforms all baseline models in 18 out of 20 experiments. We will show later that cycle-consistency losses are crucial in the high-dimensional Home-Office dataset. We speculate that due to the increased complexity in images and the generation process, the transformation-based BiOST is not as good as in digits benchmarks.

\begin{table}[htbp]
\caption{Experimental results of Product$\rightarrow$Clip Art. The evaluation metric is AUROC in percent.}
\resizebox{\columnwidth}{!}{\begin{tabular}{c|cccccccc}
\toprule
Model & DSVDD & OCSVM (S+T) & OCSVM (T) & IF (S+T) & IF (T) & AGT & BiOST & \textbf{IRAD} \\
\midrule
Bike       & 51.1$\pm$2.7 & 50.0$\pm$0.1 & 50.0$\pm$0.1 & 45.5$\pm$1.8 & 57.7$\pm$3.9          & 55.7$\pm$2.6 & 52.7$\pm$0.8 & \textbf{85.7$\pm$2.8}\\
Calculator & 53.4$\pm$8.5 & 49.4$\pm$0.7 & 50.0$\pm$0.0 & 46.4$\pm$1.5 & \textbf{81.5$\pm$3.9} & 79.7$\pm$4.2 & 65.2$\pm$1.0 & 79.2$\pm$1.8 \\
Drill      & 53.5$\pm$4.8 & 47.1$\pm$1.5 & 48.2$\pm$1.2 & 58.8$\pm$2.8 & 63.6$\pm$6.0          & 54.5$\pm$3.3 & 47.0$\pm$0.5 & \textbf{71.2$\pm$5.5}  \\
Hammer     & 50.3$\pm$1.7 & 49.3$\pm$0.7 & 49.2$\pm$0.5 & 56.8$\pm$1.1 & 61.9$\pm$3.4          & 64.4$\pm$2.3 & 43.7$\pm$0.9 & \textbf{77.0$\pm$6.0}\\
Kettle     & 44.3$\pm$6.5 & 48.7$\pm$0.7 & 47.7$\pm$1.7 & 57.0$\pm$2.1 & 57.7$\pm$3.3          & 56.3$\pm$3.2 & 47.7$\pm$1.5 & \textbf{70.0$\pm$4.9}\\
Knives     & 64.3$\pm$4.3 & 48.7$\pm$0.9 & 49.5$\pm$0.6 & 36.1$\pm$2.7 & 67.8$\pm$5.0          & 68.9$\pm$3.9 & 63.1$\pm$1.5 & \textbf{70.3$\pm$3.5}\\
Pan        & 49.2$\pm$5.8 & 49.9$\pm$0.5 & 50.0$\pm$0.0 & 59.8$\pm$1.3 & 60.0$\pm$5.3          & 56.4$\pm$3.9 & 49.3$\pm$1.5 & \textbf{72.8$\pm$3.7}\\
Paper      & 51.4$\pm$1.9 & 49.0$\pm$0.8 & 48.7$\pm$0.7 & 58.4$\pm$3.1 & 61.1$\pm$5.6          & 63.8$\pm$3.8 & 45.1$\pm$2.6 & \textbf{61.8$\pm$0.8} \\
Scissors   & 49.0$\pm$8.7 & 48.5$\pm$0.6 & 48.5$\pm$1.3 & 59.0$\pm$1.1 & 62.9$\pm$3.0          & 66.5$\pm$3.7 & 38.6$\pm$0.8 & \textbf{70.0$\pm$3.3} \\
Soda       & 48.8$\pm$5.8 & 49.9$\pm$0.4 & 50.0$\pm$0.1 & 50.9$\pm$1.8 & 56.4$\pm$7.8          & 57.3$\pm$8.7 & 56.9$\pm$0.8 & \textbf{63.2$\pm$4.9}
\end{tabular}}
\label{tab:p-c}
\vspace{-1em}
\end{table}

\begin{table}[htbp]
\caption{Results of Clip Art$\rightarrow$Product with the best numbers in bold. The metric is AUROC\%.}
\resizebox{\columnwidth}{!}{\begin{tabular}{c|cccccccc}
\toprule
Model & DSVDD & OCSVM (S+T) & OCSVM (T) & IF (S+T) & IF (T) & AGT & BiOST & \textbf{IRAD} \\
\midrule
Bike       & 49.4$\pm$11.6 & 46.2$\pm$1.2 & 46.5$\pm$2.2 & 51.4$\pm$2.0 & 65.5$\pm$3.69         & 54.0$\pm$2.5 & 43.0$\pm$0.6 & \textbf{90.3$\pm$2.6}\\
Calculator & 48.6$\pm$6.7  & 50.0$\pm$0.1 & 50.0$\pm$0.1 & 46.3$\pm$3.0 & 57.6$\pm$6.3          & 56.5$\pm$5.2 & 69.0$\pm$0.6 & \textbf{82.2$\pm$1.8} \\
Drill      & 52.8$\pm$9.5  & 50.0$\pm$0.1 & 50.0$\pm$0.1 & 34.4$\pm$1.3 & 64.4$\pm$5.1          & 33.9$\pm$2.1 & 66.4$\pm$0.7 & \textbf{73.0$\pm$5.4}  \\
Hammer     & 44.7$\pm$9.0  & 47.8$\pm$0.6 & 48.7$\pm$0.5 & 81.9$\pm$1.5 & 80.0$\pm$1.1          & 79.4$\pm$1.2 & 50.1$\pm$0.7 & \textbf{84.5$\pm$2.8} \\
Kettle     & 49.1$\pm$11.1 & 50.0$\pm$0.1 & 50.0$\pm$0.1 & 45.4$\pm$1.5 & 55.6$\pm$5.0          & 52.0$\pm$3.1 & 63.0$\pm$1.0 & \textbf{75.8$\pm$8.5}\\
Knives     & 57.2$\pm$1.8  & 48.1$\pm$1.2 & 49.4$\pm$0.8 & 48.7$\pm$1.8 & 36.0$\pm$1.5          & 47.3$\pm$3.3 & 48.8$\pm$2.2 & \textbf{63.9$\pm$2.4}\\
Pan        & 50.2$\pm$7.6  & 50.0$\pm$0.1 & 50.0$\pm$0.0 & 45.0$\pm$2.0 & 60.9$\pm$2.4          & 48.4$\pm$3.6 & 57.7$\pm$1.4 & \textbf{76.0$\pm$4.5}\\
Paper      & 48.0$\pm$9.3  & 50.0$\pm$0.1 & 50.0$\pm$0.1 & 68.0$\pm$2.0 & 70.6$\pm$4.0          & \textbf{74.9$\pm$ 3.2}& 27.4$\pm$4.0 & 67.4$\pm$3.4 \\
Scissors   & 51.3$\pm$10.1 & 49.5$\pm$0.4 & 49.5$\pm$0.7 & 63.0$\pm$0.9 & 59.0$\pm$1.5          & 65.0$\pm$1.2 & 56.4$\pm$0.6 & \textbf{68.9$\pm$4.0} \\
Soda       & 52.9$\pm$12.0 & 48.5$\pm$1.0 & 47.9$\pm$1.2 & 34.1$\pm$2.5 & 51.0$\pm$13           & 48.0$\pm$9.0 & 50.2$\pm$1.2  & \textbf{53.3$\pm$1.8}
\end{tabular}}
\label{tab:c-p}
\vspace{-1em}
\end{table}

\subsection{Bounds for the Joint Error and the Generalization Error}
\label{sec:thm}
Recent theoretical works on classification domain adaptation discover that minimizing the empirical error on the source domain can be detrimental for the model's performance in the target domain \citep{zhao19a}. We observe the same phenomenon in domain-adaptation AD that overtraining IRAD leads to less accurate detection, as shown in \Cref{fig:epochs}. The adaptation performance first grows and gradually decreases after around 5 epochs. We derive an information-theoretic lower bound of the joint error (\Cref{thm:lower}) to explain this phenomenon.


We start with definitions and notations. Let $\gD^{Y_S}$ and $\gD^{Y_T}$ 
denote the marginal label distribution in the source and target domain. The projection from the data space $X$ to the latent invariant representation space $Z$, induced by $E_{sh}$ in the case of IRAD, is denoted as $g$. The hypothesis (labeling) function $h$ is shared between two domains that map invariant representations $Z$ to predictions $\hat{Y}$. For IRAD, the hypothesis $h$ is induced by the IF learned on the invariant representations (IF learns the anomaly function). To ease the proof process, we assume the anomaly scores are transferred to classification probabilities, for example by applying a threshold. 

\begin{wrapfigure}[19]{R}{0.4\textwidth}
\vspace{-2.6em}
\includegraphics[width=0.4\textwidth]{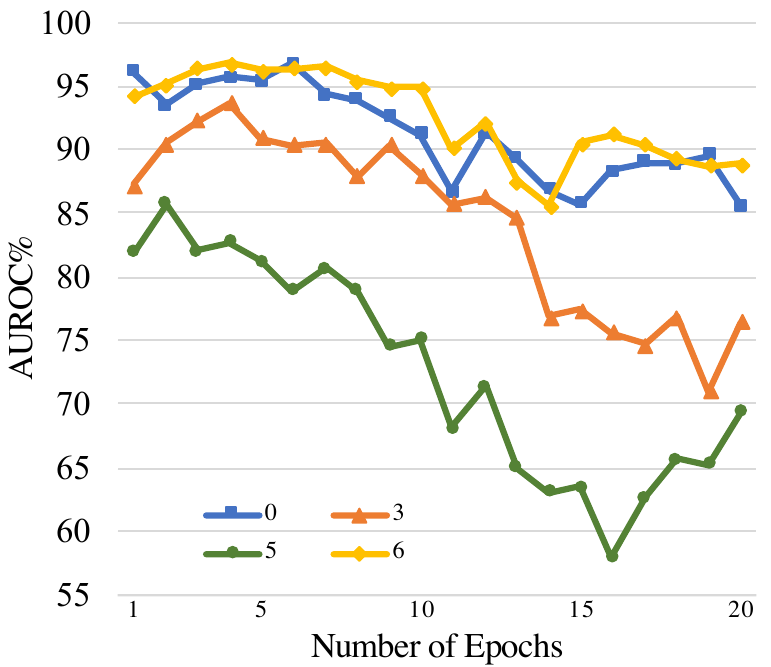}
\caption{Overtraining to minimize the source domain error hurts the performance on the target domain  (experiments conducted on MNIST$\rightarrow$USPS).}
\label{fig:epochs}
\end{wrapfigure}

The above process can be denoted as the Markov chain 
$X \stackrel{g}{\longrightarrow} Z \stackrel{h}{\longrightarrow} \hat{Y}$
\citep{ganin2016domain, zhao19a}. Let $\djs$ denote the JS distance which is the square root of JS divergence \citep{endres2003new}. Let $\varepsilon_{S}\left(h \circ g\right)$ and $\varepsilon_{T}\left(h \circ g\right)$ denote the error of the learned model in the source and target domain respectively. Then we have the following theorem on the lower bound for joint error (the proofs of theorems are provided in the appendix):
\begin{theorem}
\label{thm:lower}
Suppose the Markov chains hold, then a lower bound for the joint error on the source and target domains is: 
\vspace{-1.5em}

\begin{equation}
 \varepsilon_{S}\left(h \circ g\right)+\varepsilon_{T}\left(h \circ g\right) \geq \frac{1}{2}
 \djs(\gD^{Y_S}, \gD^{Y_T})^2
\end{equation}
\end{theorem}

\vspace{-0.5em}

\textbf{Remark:} Since the definitions of normal data are different in source and target domains,
$\djs(\gD^{Y_S}, \gD^{Y_T}) > 0$.
This term is dataset-intrinsic and independent of the learning models. The lower bound explains the phenomenon in \Cref{fig:epochs}: overtraining to minimize $\varepsilon_S$ actually increases the error on the target domain $\varepsilon_T$. Learning without adaptation (e.g. IF (S+T)) can have small $\varepsilon_S$ but still large error in the target domain. This lower bound also holds for other domain adaptation anomaly detection methods that use invariant representations. This theorem reveals that to have a well-performing model on the target domain, one needs to balance between learning effective invariant representations for accurate AD on the source domain while accommodating the target domain data. This trade-off is hard to avoid and is a consequence of our assumption that the data for $T$ is insufficient for accurate training of the model. 
We use cross-validation to estimate the optimal number of training epochs as mentioned before.


We also derive an upper bound for the generalization error. Let $f_S$, $f_T$ be the true labeling function for the source and target domains respectively. Let $\widehat{D}_S$ and $\widehat{D}_T$ denote the empirical source and target distributions from source domain samples $\mathbf{S}$
and target domain samples $\mathbf{T}$ of size $n_t$. Then:

\begin{theorem}
\label{thm:upper}
For a hypothesis space $\gH \subseteq [0, 1]^{X}$, $\forall h \in \gH$, $\forall \delta>0$, w.p.\ at least $1-\delta$:

\vspace{-1.7em}

\begin{align*}
\varepsilon_{T}(h) \leq & \;\widehat{\varepsilon}_{S}(h)+d_{\tilde{\gH}}\left(\widehat{\mathcal{D}}_{S}, \widehat{\mathcal{D}}_{T}\right)+2 \Rad_{\mathbf{S}}(\gH)+2 \Rad_{\mathbf{S}}(\tilde{\gH}) + 2 \Rad_{\mathbf{T}}(\tilde{\gH})\\[-2pt]
&+\min \left\{\mathbb{E}_{\mathcal{D}_{S}}\left[\left|f_{S}-f_{T}\right|\right], \mathbb{E}_{\mathcal{D}_{T}}\left[\left|f_{S}-f_{T}\right|\right]\right\} + O(\sqrt{\log (1 / \delta) / n_t})
\end{align*}
where $\tilde{\gH} := \left\{\operatorname{sgn}\left(\left|h(\mathbf{x})-h^{\prime}(\mathbf{x})\right|-t\right) | \, h, h^{\prime} \in \gH, t \in[0,1]\right\}$. $\Rad_{\mathbf{S}}$ denotes the empirical Rade\-ma\-cher complexity w.r.t.\ samples $\mathbf{S}$ (see the formal definition in the appendix).
\end{theorem}

\vspace{-0.5em}

\textbf{Remark:} this bound is formed by the following 5 components (left to right): (1) empirical error on $S$, (2) distance between the training sets of $S$ and $T$, (3) complexity measures of $\gH$ and $\tilde{\gH}$, (4) differences in labels between source and target, (5) error caused by limited target domain samples.

\section{Discussion}

\label{sec:abl}
\textbf{Ablation Study of Objective Functions.} To better understand the objective functions of IRAD, we conduct the following ablation studies by removing certain terms in the training process. We first investigate \Cref{eq:sim} that encourages the similarity between the shared encodings of the source and target data. Ideally, the shared encodings $E_{sh}(\vx_{tgt})$ and $E_{sh}(\vx_{src})$ should reside in the same region. For the purpose of illustration, we visualize $E_{sh}(\vx_{tgt})$ and $E_{sh}(\vx_{src})$ in 2D by linear PCA as shown in the left sub-figures of \Cref{fig:sim}(a) and \Cref{fig:sim}(b). With the similarity objective function in \Cref{eq:sim}, $E_{sh}(\vx_{tgt})$ and $E_{sh}(\vx_{src})$ are close in the latent space (\Cref{fig:sim}(a)); without \Cref{eq:sim}, $E_{sh}(\vx_{tgt})$ and $E_{sh}(\vx_{src})$ are apart (\Cref{fig:sim}(b)). We also plotted the magnitude of normalized inner products between 10 $E_{sh}(\vx_{tgt})$ and $E_{sh}(\vx_{src})$ in the right sub-figures in \Cref{fig:sim}(a) and \Cref{fig:sim}(b). The results indicate that optimizing with \Cref{eq:sim} indeed makes $E_{sh}(\vx_{tgt})$ and $E_{sh}(\vx_{src})$ close numerically.

\begin{figure}
  \includegraphics[width=\columnwidth]{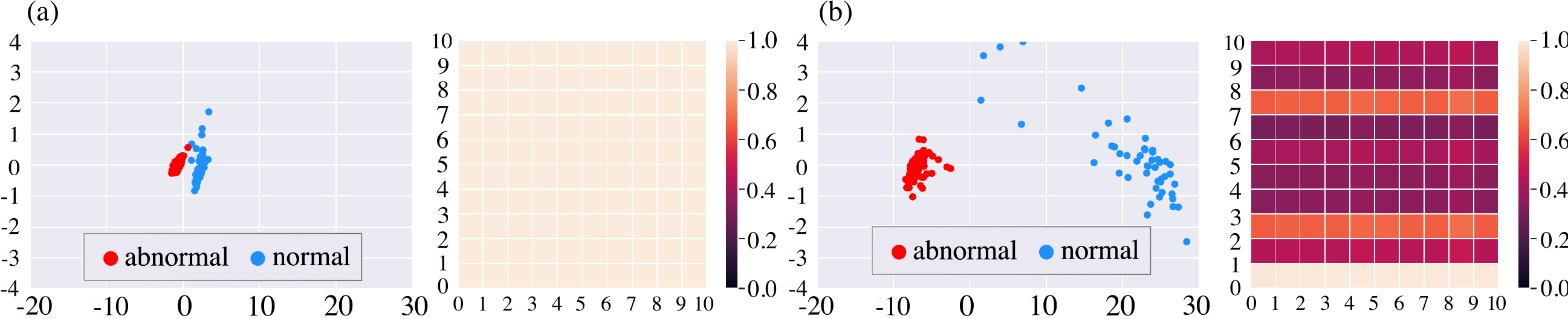}
  \vspace{-1.5em}  
  \caption{Ablation study (digit 7, MNIST$\rightarrow$USPS) on the similarity objective function in \Cref{eq:sim}. Part (a) is training with \Cref{eq:sim}: the left figure shows the 2D linear PCA projection of $E_{sh}(\vx_{tgt})$ and $E_{sh}(\vx_{src})$. The right sub-figure shows the magnitude of normalized inner products of ten randomly selected $E_{sh}(\vx_{src})$ and $E_{sh}(\vx_{tgt})$. Part (b) is trained without \Cref{eq:sim}. $E_{sh}(\vx_{tgt})$ and $E_{sh}(\vx_{src})$ are geometrically and numerically apart from each other in this case.}
  \label{fig:sim}
  \vspace{-0.5em}
\end{figure}

We further study the cycle-consistency losses in \Cref{eq:cycle}. We find them critical in Office-Home dataset evaluations. Training without them can lead to more than 10\% decrease in performance. We visualize the extracted features from the normal and abnormal target data, $E_{sh}(\vx_{nor})$ and $E_{sh}(\vx_{abn})$, in 2D with PCA. Ideally, $E_{sh}(\vx_{nor})$ and $E_{sh}(\vx_{abn})$ should be separated so the abnormal can be detected. This is what we observe when training with the full model (the first column of \Cref{fig:norm_abn}(a)). However, if optimized without \Cref{eq:cycle}, encoded normal and abnormal data are mixing together (the second column in \Cref{fig:norm_abn}(a)). We also investigate term $\vx_{rnd}$ in \Cref{eq:adv}. Removing $\vx_{rnd}$ from the adversarial training results in $E_{sh}(\vx_{nor})$ and $E_{sh}(\vx_{abn})$ mingling together (the third column of \Cref{fig:norm_abn}(a)). We conjecture that for high dimensional data like images, it is challenging for the discriminator to form an effective decision boundary \citep{rcgan}, therefore additional regularization terms ($\vx_{rnd}$) and objective functions (cycle-consistent losses) are helpful for modeling the normal data distribution. 

\begin{figure}
  \includegraphics[width=\columnwidth]{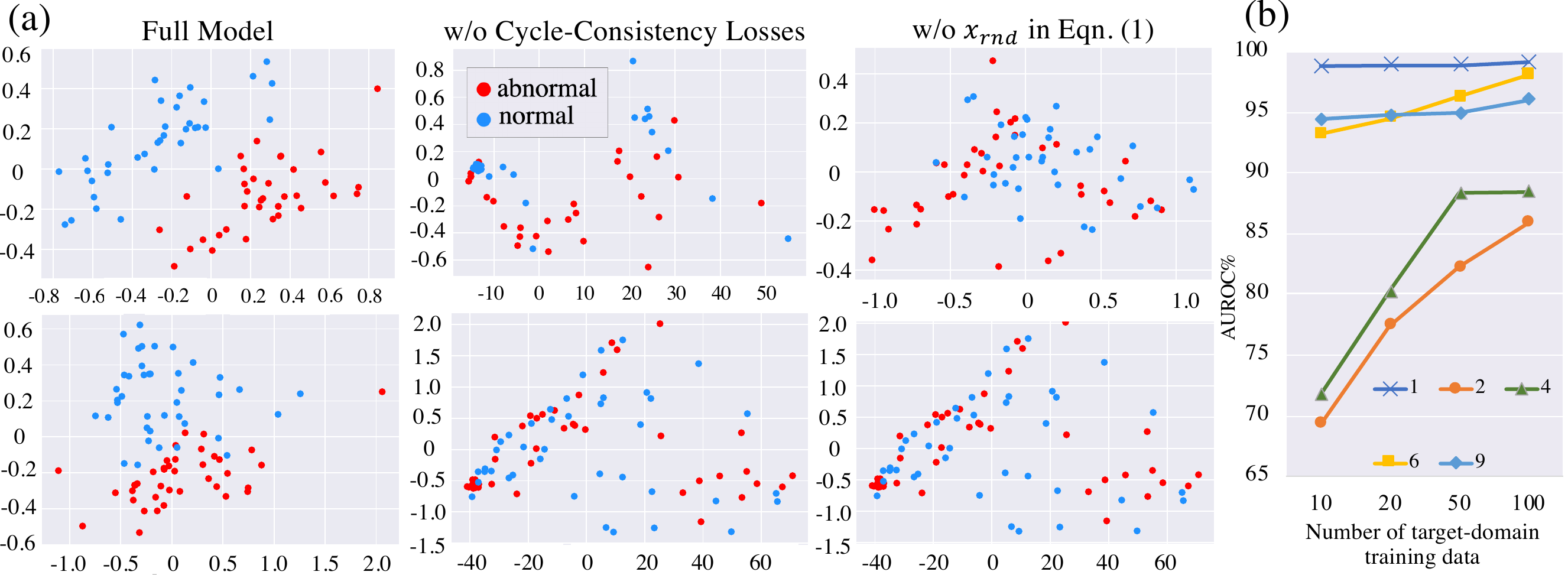}
  \vspace{-1.5em}    
  \caption{(a) Invariant representations of normal (blue) and abnormal (red) target domain data. The first and second rows are ``Calculator'' and ``Pan'' (``Product''$\rightarrow$``Clip Art''). The first column is training with the full model and normal and abnormal encodings are well separated. The second column is training without the cycle-consistency losses. Third column is from removing the term $\vx_{rnd}$ in \Cref{eq:adv}. Normal and abnormal data in the two later cases are mixing, making detection hard. (b) AUROC on MNIST$\rightarrow$USPS with different numbers of target-domain training data. We only present 5 digits due to the space limit; the full results are provided in the appendix.}
  \label{fig:norm_abn}
  \vspace{-1em}
\end{figure}

\textbf{Effects of the number of target domain training data.} We investigate the performance of IRAD w.r.t.\ the number of available target-domain training data $n_t$. The results are presented in \Cref{fig:norm_abn}(b) with $n_t = 10, 20, 50, 100$. IRAD is able to leverage more target data to achieve better performance.




\vspace{-0.5em}
\section{Conclusion}
\vspace{-0.5em}

We studied the domain adaptation problem in anomaly detection. The proposed method IRAD first learns invariant representations between the source and target domains. This is achieved by isolating the shared encodings from domain-specific encodings through adversarial learning and enforcing subspace similarity/dissimilarity. The domain-invariant representations are then used train an anomaly detection model in the target domain. We show that IRAD significantly outperform baseline models in most experiments on digits and high-dimensional object recognition datasets. We prove a lower bound for the joint error and an generalization upper bound. Experimental observations corroborate our theoretical results. 

\section*{Broader Impact}
As described and benchmarked in this paper, IRAD should have neutral societal consequences. The experimental benchmarks in the paper are standard academic datasets and do not involve human subjects. It is possible however to apply these methods for video security and surveillance. Although in principle there is no explicit bias in our algorithms, it is possible that the outcome of the prediction be biased. This may occur if there is a bias in the training set or if the source and target domains are biased. Generally speaking, anomaly detection algorithms applied to human subjects should be benchmarked and validated extensively to avoid racial/gender/age/religion/disability bias and other types of biases.



\bibliography{neurips_2020}
\bibliographystyle{neurips_2020}

\newpage
\appendix
\section{Experiments with different numbers of target-domain training data}

\begin{figure}[!htbp]
  \centering
  \includegraphics[width=0.5\columnwidth]{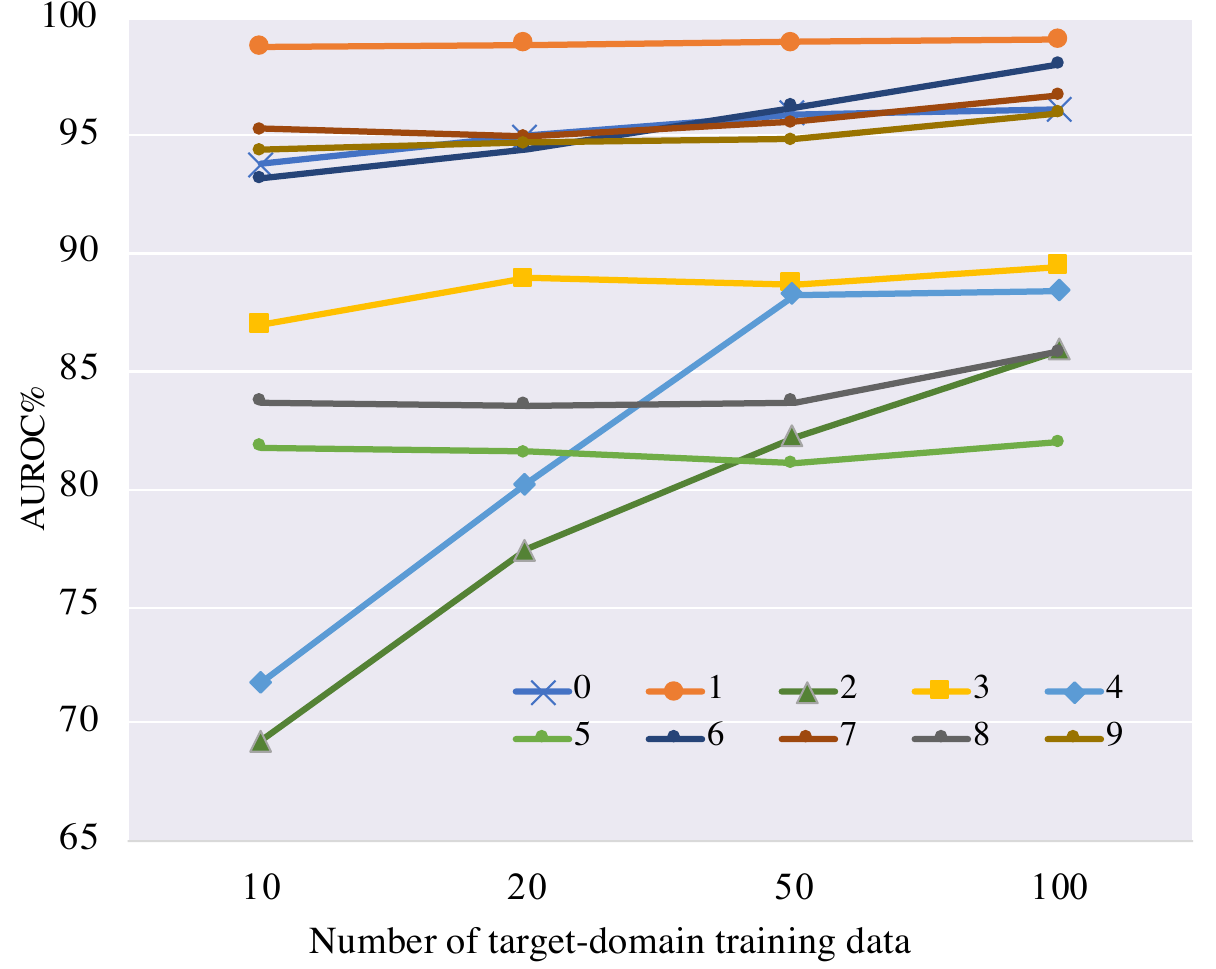}
  \caption{AUROC on MNIST-USPS experiments of all ten categories with the number of target-domain training data $n_t = 10, 20, 50, 100$. The model performance increases as more target-domain data are available for training.}
  \label{fig:appd_1}
\end{figure}

\section{Examples of images in Office-Home dataset for evaluation}
\begin{figure}[!htbp]
  \centering
  \includegraphics[width=\columnwidth]{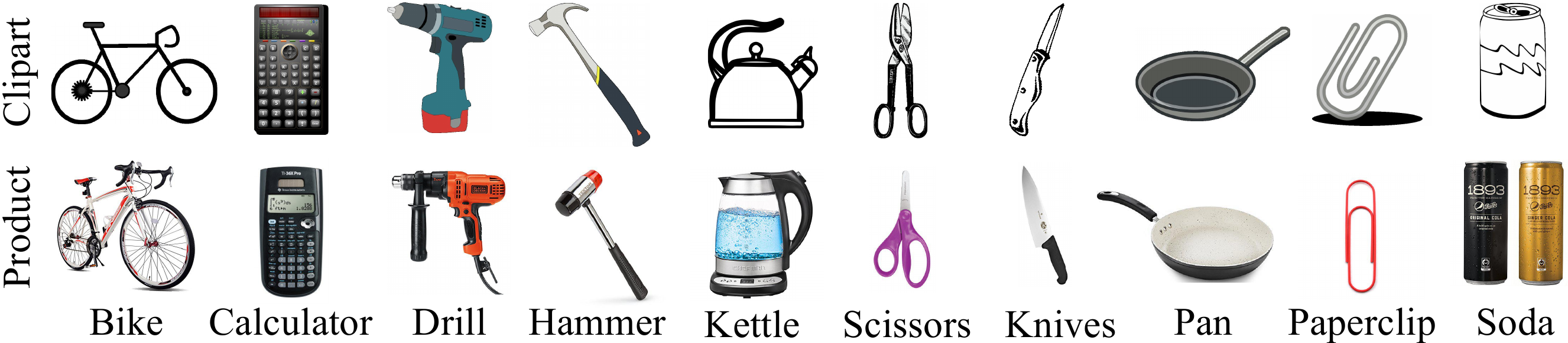}
  \caption{Examples of ten categories in Clip Art and Product domain from Office-Home dataset.}
  \label{fig:dataset}
\end{figure}

\section{Proof of Information-Theoretic Lower Bound:}
The proof for \Cref{thm:lower} is as follows:
\begin{proof}
$Y_S$ are defined as the labeling function from $X$ to $Y$ for the source domain and $Y_T$ as the map for the target domain. We assume that $Y$ is 1 when the data is anomalous and 0 otherwise. Since the JS distance is a metric, we have the following inequality:
\begin{equation}
    \label{eq:trineq}
    \djs(\gD^{Y_S}, \gD^{Y_T}) \leq \djs(\gD^{Y_S}, \gD^{\hat{Y}}) + \djs(\gD^{\hat{Y}}, \gD^{Y_T})
\end{equation}
We define $\varepsilon_S(h \circ g) = \varepsilon_S(\hat{Y})$ as 
$$\varepsilon_S(h \circ g)
= \mathbb E_X(|Y_S(X) - h \circ g(X)|)$$
and similarly for $\varepsilon_T(h \circ g)$. We can bound $\djs(\gD^{Y_S}, \gD^{\hat{Y}})$ by $\sqrt{\varepsilon_S(h \circ g)}$ (\cite{lin1991divergence,zhao19a})
\begin{align*}
    \djs(\gD^{Y_S}, \gD^{\hat{Y}}) &=
    \sqrt{ D_\text{JS}(\gD^{Y_S}, \gD^{\hat{Y}}) }
    \le \sqrt{ \frac{1}{2} \| \gD^{Y_S} - \gD^{\hat{Y}} \|_1 } \\
    &= \sqrt{ \frac{1}{2} \Big(
    \big|\text{Pr}(Y_S=0) - \text{Pr}(\hat{Y}=0) \big|
    +
    \big|\text{Pr}(Y_S=1) - \text{Pr}(\hat{Y}=1) \big|
    \Big)} \\
    & = \sqrt{ \big|\text{Pr}(Y_S=1) - \text{Pr}(\hat{Y}=1) \big| }
    = \sqrt{ \big|\mathbb E_X(Y_S) - \mathbb E_X(\hat{Y}) \big| } \\
    &\le \sqrt{\varepsilon_S(h \circ g)}
\end{align*}

With \Cref{eq:trineq}, we get:
\begin{equation}
\label{eq:error}
    \djs(\gD^{Y_S}, \gD^{Y_T})
    \leq \sqrt{\varepsilon_S} + \sqrt{\varepsilon_T}
\end{equation}
This can be rewritten as:
$$
    \varepsilon_S\left(h \circ g\right) + \varepsilon_T\left(h \circ g\right) 
    \geq 
    \frac{1}{2} \; 
    \djs(\gD^{Y_S}, \gD^{Y_T})^2
$$
\end{proof}

\section{Proof of Generalization Upper Bound:}

We start with introductions of notations and definitions. Recall $$\tilde{\gH}:=\left\{\operatorname{sgn}\left(\left|h(\mathbf{x})-h^{\prime}(\mathbf{x})\right|-t\right) | \, h, h^{\prime} \in \gH, t \in[0,1]\right\}$$
Let $\widehat{D}$ denote the empirical distribution from samples $x \sim \gD$ of size $n$. The empirical Rademacher complexity is defined as follows:
\begin{definition}[Empirical Rademacher Complexity]
Let $\gH$ be a family of functions mapping from $X$ to $[a, b]$. Let $\mathbf{S} = \{\vx_i\}_{i = 1}^n$ denote a fixed sample of size $n$ with elements in $X$. Then, the empirical Rademacher complexity of $\gH$ with respect to the sample $X$ is defined as:
$$
\Rad_\mathbf{S}(\mathcal{H}):=\mathbb{E}_{\boldsymbol{\sigma}}\left[\sup _{h \in \mathcal{H}} \frac{1}{n} \sum_{i=1}^{n} \sigma_{i} h\left(\mathbf{x}_{i}\right)\right]
$$
where $\sigma=\left\{\sigma_{i}\right\}_{i=1}^{n}$ and $\sigma_{i}$ are i.i.d.\ uniform random variables taking values in $\{+1,-1\}$.
\end{definition}

We then have the following lemmas.
\begin{lemma}[\citet{zhao19a}]
\label{lemma1}
Let $\gH \subseteq [0, 1]^X$, then for all $\delta > 0$, w.p.\ at least $1 - \delta$, the following inequality holds for all $h\in \gH$:\ $\varepsilon_{S}(h) \leq \widehat{\varepsilon}_{S}(h)+2 \Rad_{\mathbf{S}}(\gH)+3 \sqrt{\log (2 / \delta)/2n}$, where $n$ is the number of samples in $\mathbf{S}$.
\end{lemma}

\begin{lemma}[\citet{zhao19a}]
\label{lemma2}
$\forall \delta > 0$, w.p. at least $1 - \delta$, the following inequality holds:
$$d_{\tilde{\mathcal{H}}}(\mathcal{D}, \widehat{\mathcal{D}}) \leq 2 \Rad_\mathbf{S}(\tilde{\mathcal{H}})+3 \sqrt{\log (2 / \delta) / 2 n}.$$
\end{lemma}

With \Cref{lemma1} and \Cref{lemma2}, we can derive:
\begin{lemma}
\label{lemma3}
For $\forall \delta>0$, w.p.\ at least $1-\delta,$ for $\forall h \in \tilde{\gH}$:
$$
d_{\tilde{\gH}}(\mathcal{D}_S, \mathcal{D}_T) \leq d_{\tilde{\gH}}(\widehat{\mathcal{D}}_S, \widehat{\mathcal{D}}_T)+ 2 \Rad_{\mathbf{S}}(\tilde{\gH}) + 2 \Rad_{\mathbf{T}}(\tilde{\gH}) + 3 \sqrt{\log (4 / \delta)/2n} + 3 \sqrt{\log (4 / \delta)/2n_t}$$
\end{lemma}
\begin{proof}
The triangular inequality of $d_{\tilde{\gH}}(\cdot, \cdot)$ is written as: 
$$d_{\tilde{\mathcal{H}}}\left(\mathcal{D}, \mathcal{D}^{\prime}\right) \leq d_{\tilde{\mathcal{H}}}(\mathcal{D}, \widehat{\mathcal{D}})+d_{\tilde{\mathcal{H}}}(\widehat{\mathcal{D}}, \widehat{\mathcal{D}}^{\prime})+d_{\tilde{\mathcal{H}}}(\widehat{\mathcal{D}}^{\prime}, \mathcal{D}^{\prime}).$$

By \Cref{lemma2}, it follows that with probability $\geq 1 - \delta/2$, the following two inequalities hold:
\begin{align*}
d_{\tilde{\mathcal{H}}}(\mathcal{D}_S, \widehat{\mathcal{D}}_S) &\leq 2 \Rad_{\mathbf{S}}(\tilde{\mathcal{H}})+3 \sqrt{\log (4 / \delta) / 2 n}\\
d_{\tilde{\mathcal{H}}}(\mathcal{D}_T, \widehat{\mathcal{D}}_T) &\leq 2 \Rad_{\mathbf{T}}(\tilde{\mathcal{H}})+3 \sqrt{\log (4 / \delta) / 2 n_t}
\end{align*}
These two inequalities can be combined as with a union bound to obtain the inequality in the lemma.
\end{proof}

\begin{lemma}[\citet{zhao19a}]
\label{lemma4}
Let $\left\langle\mathcal{D}_{S}, f_{S}\right\rangle$ and $\left\langle\mathcal{D}_{T}, f_{T}\right\rangle$ be the source and target domains respectively. For any function class $\mathcal{H} \subseteq[0,1]^{\mathcal{X}},$ and $\forall h \in \mathcal{H},$ the following inequality holds:
\begin{align*}
\varepsilon_{T}(h) \leq  \varepsilon_{S}(h)+d_{\tilde{\mathcal{H}}}\left(\mathcal{D}_{S}, \mathcal{D}_{T}\right) +\min \left\{\mathbb{E}_{\mathcal{D}_{S}}\left[\left|f_{S}-f_{T}\right|\right], \mathbb{E}_{\mathcal{D}_{T}}\left[\left|f_{S}-f_{T}\right|\right]\right\}
\end{align*}
\end{lemma}
Finally, the proof of generalization upper bound \Cref{thm:upper} is given as:
\begin{proof}
Following \Cref{lemma4}, we have:
$$\varepsilon_{T}(h) \leq \varepsilon_{S}(h)+d_{\tilde{\mathcal{H}}}\left(\mathcal{D}_{S}, \mathcal{D}_{T}\right) +\min \left\{\mathbb{E}_{\mathcal{D}_{S}}\left[\left|f_{S}-f_{T}\right|\right], \mathbb{E}_{\mathcal{D}_{T}}\left[\left|f_{S}-f_{T}\right|\right]\right\}$$
The probabilistic bounds for $\varepsilon_{S}(h)$ are given in \Cref{lemma1} and \Cref{lemma3}. Applying them to the inequality above finishes the proof. The term $O(\sqrt{\log (4 / \delta)/2n})$ goes away since $n_T \ll n.$
\end{proof}

\end{document}

%% file: math_commands.tex
\usepackage{amsmath,amsfonts,bm}

\def\eqref#1{equation~\ref{#1}}

\def\1{\bm{1}}

\def\vx{{\bm{x}}}

\def\vz{{\bm{z}}}

\DeclareMathAlphabet{\mathsfit}{\encodingdefault}{\sfdefault}{m}{sl}
\SetMathAlphabet{\mathsfit}{bold}{\encodingdefault}{\sfdefault}{bx}{n}

\def\gD{{\mathcal{D}}}

\def\gH{{\mathcal{H}}}

\def\gN{{\mathcal{N}}}

\newcommand{\djs}{d_{\mathrm{JS}}}
\newcommand{\Rad}{\operatorname{Rad}}

\newcommand{\E}{\mathbb{E}}

\newcommand{\R}{\mathbb{R}}

